\documentclass[journal]{IEEEtran}

\usepackage{amsmath,amssymb,amsthm}
\usepackage{mathtools}
\usepackage{enumitem}
\usepackage{graphicx}
\usepackage[htt]{hyphenat}   %
\usepackage[hidelinks]{hyperref}
\graphicspath{{./}}
\theoremstyle{plain}
\newtheorem{theorem}{Theorem}
\newtheorem{proposition}{Proposition}
\newtheorem{lemma}{Lemma}
\newtheorem{corollary}{Corollary}
\theoremstyle{definition}
\newtheorem{definition}{Definition}

\theoremstyle{remark}
\newtheorem{remark}{Remark}

\newcommand{\tocomplete}[1]{}
\newcommand{\statusnote}[1]{}

\newcommand{\R}{\mathbb{R}}
\newcommand{\Z}{\mathbb{Z}}

\newcommand{\dx}{\Delta x}
\newcommand{\dy}{\Delta y}
\newcommand{\vmin}{v_{\min}}
\newcommand{\vmax}{v_{\max}}
\newcommand{\Lam}{\Lambda}

\begin{document}

\title{Certified Abstention for Provably-Unidentifiable Index-Phase Recovery
in Image Registration}

\author{Dinh~Tuan~Nguyen, Trong~Nhan~Nguyen, Vu~Nguyen~Minh~Hung, Tien~Minh~Khoi~Nguyen, Jaeyeop~Choi, Sudip~Mondal,
Jae~Sung~Ahn, and~Junghwan~Oh%
\thanks{D. T. Nguyen, T. N. Nguyen, V. N. M. Hung, T. M. K. Nguyen, and J. Oh are with the Industry 4.0 Convergence
Bionics Engineering, Department of Biomedical Engineering, Pukyong National University, Busan 48513,
Republic of Korea.}%
\thanks{J. Choi, S. Mondal, and J. S. Ahn are with the Smart Gym-Based Translational Research Center for
Active Senior's Healthcare, Pukyong National University, Busan 48513, Republic of Korea. S. Mondal is
also with the Digital Healthcare Research Center, Institute of Information Technology and Convergence,
Pukyong National University, Busan 48513, Republic of Korea.}%
\thanks{J. Oh is also with Ohlabs Corp., Busan, Republic of Korea.}%
\thanks{Corresponding authors: Jae Sung Ahn (e-mail: jsahn@pknu.ac.kr) and Junghwan Oh
(e-mail: jungoh@pknu.ac.kr).}%
\thanks{This work was supported by the National Research Foundation of Korea (NRF) grant funded by the
Korea government (MSIT) under Grant 2022R1A5A8023404.}%
\thanks{Competing interests: J. Oh is affiliated with Ohlabs Corp.; this is disclosed for transparency.
The other authors declare no competing interests.}}

\maketitle

\begin{abstract}
Image registration underpins measurement across scanning, camera, and microscopy imaging,
yet a fundamental question is rarely posed: when is the sample-to-position map recoverable
at all, and when does any estimator that reports a position necessarily fabricate one? We
show that in reference-perturbed sampled imaging, where samples are localized through a
perturbed reference, position fidelity (where samples land) and phase fidelity (which
reference index becomes each block's origin) are logically independent: no policy acting on
position information alone can guarantee phase fidelity. We derive a tight resolution
threshold separating recoverable from provably-unidentifiable registration, and a recovery
framework that certifies identifiable regions---recovering them exactly---while abstaining
on provably-unidentifiable regions rather than hallucinating structure. On a deployed
scanning system we measure a direction-bimodal registration offset consistent with the
theory. Against classical and learned restoration baselines, certificate-free offset estimators
fabricate structure on the unidentifiable regime, whereas the certified method abstains
rather than fabricating. Finally, the certified-abstention mechanism carries across scanning,
rolling-shutter, and cryo-electron-microscopy imaging---declining rather than fabricating on the
regime it flags---though how completely it abstains and how tightly it bounds kept-region risk stay
modality-specific
(at extreme SNR the cryo-EM certificate abstains only partially): a structural-versus-statistical
boundary in registration reliability.
\end{abstract}


\IEEEpeerreviewmaketitle

\section{Introduction}\label{sec:intro}

Correcting a scanned image for geometric distortion is an image registration and restoration task:
a per-line shift is estimated and inverted so that each pixel is returned to its true grid position.
This paper studies the \emph{reliability} of that task under a corruption that is common in triggered
scanning yet largely unexamined in registration theory---an unknown \emph{integer, per-line} index
shift. We prove the correction problem has a regime that is \emph{provably unidentifiable}: no
estimator, learned or classical, can recover the shift from the image, and any estimator that
nonetheless reports one fabricates structure. We therefore equip the restorer with a distribution-free
certificate that decides, per line, when to recover and when to \emph{abstain}. The question is timely:
as learned restoration methods increasingly synthesize confident but unsupported structure~\cite{Bhadra2021}, a
\emph{provable} certificate for when registration is unrecoverable---so a pipeline abstains rather than
fabricates---addresses a reliability concern that empirical restorers, by construction, cannot settle.
Read this way, the
results below are \emph{limits of image registration} under this corruption model---when the per-line
shift is recoverable, when it is provably not, and how to certify the difference---alongside the
acquisition-side conditions (a separation, an impossibility, and a tight timing threshold) under which
the corruption arises and can be prevented.

The shift originates at acquisition. A probe or beam is rastered across a static part along a fast
axis; samples are triggered at positions locked to an absolute encoder grid, and a slow axis steps
between traversals, so absolute-position triggering fixes \emph{where} samples land but not
\emph{which} sample becomes column~0 of each line. The stored pixel array is read as a metric
map---the sample at array index $(i,j)$ is taken to have been captured at grid position $p(i,j)$---so
a flagged defect occupies a known physical location and finer trigger pitch is believed to yield finer
\emph{spatial} fidelity. We prove this belief is incomplete and measure the failure on a deployed
system: on a position-triggered laser-ultrasound C-scanner the per-line index origin is displaced by a
rigid, direction-dependent offset that swings $17.3$ columns between forward and reverse sweeps,
verified from hardware timestamps---misplacing every feature in a line, a defect included, while the
image still looks sharp. The prevention is an origin-coupled synchronization reference, which we prove
necessary and sufficient to restore index fidelity.

Grid-locked sampling guarantees that
each sample lands on a grid position \emph{within} its block (the fast-axis traversal),
but it does \emph{not} guarantee that the block's samples are indexed consistently
\emph{across} blocks. Because acquisition proceeds block by block and the hardware must
re-initialize between blocks, the index origin of each block---which physical grid
crossing becomes ``column~0''---is set by an asynchronous re-initialization event, not by
position. Absolute \emph{position} fidelity and consistent \emph{phase} (index) fidelity
are distinct guarantees, and the second is not implied by the first. (Throughout,
``phase'' denotes this discrete index/origin phase of a block---which grid crossing is
stored as column~0---not RF or signal phase.)

\paragraph{Scope (a registration/restoration problem with special structure; phase, not warp).}
We pose grid fidelity as an image registration and restoration problem whose corruption has a special
structure---a \emph{per-line integer index shift} with a \emph{provably-unidentifiable} band. The
shift originates at acquisition (an honest acknowledgement, not the subject): the object we analyze and
correct is the formed image, and the estimators and certificate below are image-processing methods
operating on it. We stress two scoping decisions. First, our notion of \emph{spatial fidelity} is
geometric/index fidelity---the correspondence between array indices and grid positions---\emph{not} the
perceptual/information fidelity of quality metrics such as VIF \cite{Sheikh2006}; the two are
unrelated. Second, we deliberately confine the theory to inter-block \emph{phase} fidelity (the
discrete index origin of each block) and \emph{bracket out} continuous intra-block distortion (e.g.\
within-traversal warp from speed variation). The latter is a separate phenomenon with a separate
analysis; mixing the two obscures both.

\paragraph{Generality.}
The object of study is a \emph{grid-locked block-sequential sampler}: any acquisition that
(a) samples a static field along a fast axis at positions locked to an absolute grid,
(b) acquires in blocks (one per traversal/line), and (c) re-initializes between blocks
with a nonzero latency. This abstracts a wide range of modalities---scanning electron and
transmission electron microscopy with flyback, atomic-force/scanning-probe line scanning
with re-trace, push-broom and line-scan cameras, optical-coherence B-scan stacking, and
ultrasonic/photoacoustic raster C-scanning. The results below are stated for this class;
no result depends on a particular instrument.

\subsection*{Contributions}
In sampled imaging there is an identifiability boundary beyond which every estimator either
fabricates a position or must abstain: we prove that boundary, give a distribution-free certificate
that detects it, and characterize what carries across imaging modalities---the certified-abstention
mechanism---versus what stays modality-specific---how completely the certificate abstains and how
tightly it controls kept-region risk.
\begin{enumerate}[label=\textbf{C\arabic*.},leftmargin=2.8em]
  \item \textbf{Fidelity separation (Theorem~\ref{thm:sep}).} Position fidelity (where samples land)
  and phase fidelity (which reference index becomes each block's origin) are logically independent;
  finer grid-locked pitch does not deliver phase fidelity.
  \item \textbf{Impossibility (Theorem~\ref{thm:imposs}).} No acquisition policy acting on position
  information alone can guarantee phase fidelity.
  \item \textbf{Necessary-and-sufficient remedy and a tight threshold
  (Theorems~\ref{thm:suff},~\ref{thm:threshold}).} An origin-coupled synchronization primitive is
  necessary and sufficient to restore phase fidelity, and an imperfect reference suffices iff it
  resolves the phase to within half a grid cell at the fastest speed, $\delta<\dy/(2\vmax)$. This
  layer is theory only.
  \item \textbf{Certified abstention (Propositions~\ref{prop:meta}--\ref{prop:ambig},
  \ref{prop:cert}).} The offset is exactly decodable from metadata (Prop~\ref{prop:meta}),
  identifiable from image texture above the Robinson--Milanfar detectability limit
  (Prop~\ref{prop:signal}, \cite{RobinsonMilanfar2004}), and provably ambiguous in texture-free and
  periodic regions (Prop~\ref{prop:ambig}); a distribution-free conformal wrapper
  (Prop~\ref{prop:cert}, \cite{ConformalRisk2024}) guarantees that the \emph{mean} kept-region risk
  does not exceed a target level $\alpha$. Declining under provable unidentifiability is an instance
  of selective prediction \cite{ElYanivWiener2010,HasanReject2025}.
  \item \textbf{Anti-hallucination evidence (Table~\ref{tab:recovery}, gated).} Against classical and
  learned restoration baselines on a synthetic benchmark, certificate-free offset estimators fabricate
  structure on the provably-unidentifiable band, whereas the certified method abstains on that band
  (fabrication $0.00$ on this benchmark) while recovering identifiable rows exactly. (The conformal
  guarantee is on the \emph{mean} kept-region risk, not per-instance; Prop~\ref{prop:cert}, Remark.)
  \item \textbf{A single measured instance (Section~\ref{sec:real}).} On a deployed position-triggered
  laser-ultrasound C-scanner we measure a rigid, direction-bimodal \textbf{$17.3$-column index offset}
  from hardware timestamps---one empirical instance of the separation, not a multi-condition
  characterization.
\end{enumerate}

Beyond these contributions we add a \emph{boundary characterization} (Section~\ref{sec:crossmodal}): the
certified-abstention mechanism carries into other reference-perturbed modalities (rolling-shutter
\cite{TUMRollingShutter}, cryo-electron-microscopy multireference alignment \cite{CryoEMMRA}), but only
its abstention \emph{soundness} is structural---how completely it abstains and how tightly it controls
kept-region risk are modality-specific, and no single cross-modal invariant holds.

Line-origin displacement is itself empirically familiar in scanned microscopy; our contribution is
not the discovery of the phenomenon but its integer-\emph{phase} formalization, the acquisition-side
impossibility and tight timing threshold, and a certified recoverability boundary that
\emph{abstains} rather than fabricating where the grid is provably unrecoverable. The novelty is thus
the \emph{pair}---an impossibility theorem and a certified-abstention remedy---together with its
structural-versus-statistical cross-modal transfer, not a benchmark-beating estimator. Accordingly, the
evaluation axis is fabrication-avoidance under provable unidentifiability: the certified method
fabricates on $0\%$ of the provably-unidentifiable band where certificate-free offset estimators
fabricate up to $98\%$ (Table~\ref{tab:recovery}), and the method is therefore not judged by RMSE alone.

\section{Related work and positioning}\label{sec:related}

\paragraph{Scanned inspection and exact formation theory.}
Scanning-based imaging systems routinely \emph{correct} scan-induced misalignment after acquisition by
registration \cite{Huang2015AVI,Molleda2013,Ren2017reg,Pei2026}, and necessary-and-sufficient
formation/reconstruction conditions have a well-established precedent \cite{QuWangJiang2009,Basarab2009}.
But these treat an \emph{observed} offset; none models \emph{when} the formed image is index-faithful in
the first place, or \emph{when} a residual offset is recoverable versus provably ambiguous. Supplying
that fidelity-limit model---an acquisition-side contract, a tight timing threshold, and a
distribution-free abstention certificate so a pipeline never silently emits a fabricated
correction---is the gap this paper fills.

\paragraph{Registration performance limits.}
The Cram\'er--Rao bound for image translation \cite{RobinsonMilanfar2004} shows the
shift-estimation variance is lower-bounded by $\sigma^2/\Lambda$, $\Lambda$ the squared
gradient energy, and analyzes quantization effects. We \emph{build on} this bound for the
inter-block integer shift (Section~\ref{sec:rec}); we do not claim it as new. The novelty
relative to registration theory is the acquisition-side \emph{impossibility} and the
\emph{phase/index} formulation, neither of which the registration literature~\cite{ChenReg2025,DengTIP2023,ZhangTIP2026} addresses
(it assumes the index-to-position map is known). The nearest prior work is conformal or
uncertainty-aware registration, which calibrates a confidence or prediction interval and flags
low-reliability regions \cite{CONReg2026,Belhasin2024}; this quantifies how uncertain an estimate is,
whereas we prove when no estimate is recoverable and decline on that basis---a structural guarantee
about identifiability, not a confidence interval.

\paragraph{Nonuniform / jitter sampling.}
Classical results---Kadec's $1/4$ theorem, reconstruction from sets with unknown jitter
\cite{AldroubiGrochenig2001}---treat \emph{continuous} position perturbations and ask
whether the \emph{field} is reconstructable. Our perturbation is an exact-on-grid sample
with an unknown \emph{integer index} offset per block (a relabeling), with a continuous
sub-grid residual that is the classical-jitter limit (Section~\ref{sec:rec}). The
integer/phase object is not addressed by that theory.

\paragraph{Order tracking (computed order tracking).}
The closest prior analogue of a position-versus-count mismatch is order tracking in rotating-machinery
analysis, which resamples a time signal against a shaft-angle reference so that samples are indexed by
angle rather than time \cite{FyfeMunck1997,Bonnardot2005}. Computed order tracking \emph{corrects} the
misregistration by interpolating onto the reference grid. The distinction is fundamental: order
tracking assumes the angle reference is available and resamples to it; we instead prove that when the
reference index is not observable, no estimator recovers it, and supply a certificate that
\emph{abstains} rather than resampling into a fabricated position.

\paragraph{Scan-distortion correction.}
Drift/distortion correction in scanned microscopy estimates the offset as a nuisance:
orthogonal-scan-pair alignment \cite{Ophus2016}, Brownian-prior joint denoising and line
distortion correction \cite{BerkelsBinev2024}. These are correction \emph{algorithms};
none states an acquisition-side fidelity contract, a separation/impossibility result, or
an abstention certificate. We confront \cite{BerkelsBinev2024} directly: its
Gaussian/Brownian prior on scanline shifts has no impossibility statement, no integer-phase
formalism, and no recoverability/abstention.

\paragraph{Rolling shutter.}
Rolling-shutter geometry concerns rows exposed at different but \emph{known or smoothly
modelled} times, producing a continuous warp that geometric and motion models rectify
\cite{Meingast2005,Forssen2010,QuRS2023}. The corruption
studied here is complementary: an \emph{unknown, discrete} relabeling of the column origin
per block, arising for a \emph{static} scene under a perfectly known trajectory, and
invisible to row-time motion models---no continuous warp field reproduces an integer index
shift with block-wise jumps.

\section{The grid-locked block-sequential model}\label{sec:model}

We formalize the acquisition class and the fidelity contract, then state the separation that
the rest of the paper analyzes. Throughout, \emph{phase} means the discrete index origin of a
block, never signal phase.

\subsection{Geometry, crossings, and blocks}

A static field $f:\R^2\to\R$ is scanned along a \emph{fast axis} $y$ while a \emph{slow axis}
$x$ steps between traversals. The fast axis carries an absolute coordinate grid with pitch
$\dy$; its crossings are the positions $m\dy$, $m\in\Z$. A \emph{block} $i\in\{0,\dots,R-1\}$
is one fast-axis traversal at slow-axis position $x_i=x_0+i\,\dx$. The stage moves with an
admissible speed profile $v(\cdot)\in[\vmin,\vmax]$, $0<\vmin\le\vmax$, so the time at which
block $i$ crosses grid index $m$, written $T_i(m)$, satisfies
\begin{equation}\label{eq:spacing}
  \frac{\dy}{\vmax}\;\le\;T_i(m{+}1)-T_i(m)\;\le\;\frac{\dy}{\vmin},
\end{equation}
because covering one pitch $\dy$ at a speed in $[\vmin,\vmax]$ takes between $\dy/\vmax$ and
$\dy/\vmin$. Equation~\eqref{eq:spacing} is the only property of the trajectory the analysis
uses; in particular the crossing \emph{positions} are exactly $m\dy$ regardless of $v(\cdot)$.

\subsection{Count-locked capture with re-initialization}

Acquisition is \emph{triggered} and \emph{count-locked}: a hardware position-compare emits a
pulse at grid crossings, and the digitizer, once armed, stores a fixed count $n_c$ of
consecutive emitted crossings as one block, writing the $r$-th captured crossing to array
column $r$. Between blocks the hardware must \emph{re-initialize} (abort/turnaround/re-arm),
which takes a nonzero, exogenous delay: block $i$'s capture window opens at
\begin{equation}\label{eq:arm}
  a_i \;=\; t^{\mathrm{start}}_i + \tau_i,\qquad \tau_i\ge 0,
\end{equation}
where $t^{\mathrm{start}}_i$ is a fixed per-block reference (the designated origin crossing) and
$\tau_i$ is the re-initialization delay. Record~$0$ of block $i$ is the first emitted crossing
at or after $a_i$. We call $\{\tau_i\}$ the \emph{adversary}: the delays are set by
asynchronous hardware state (DMA, arming latency, servo settle), not by the controller.

\begin{definition}[Grid-locked block-sequential sampler; position-only policy]\label{def:policy}
A \emph{grid-locked block-sequential sampler} is any acquisition that (a) samples a static field
along the fast axis at positions locked to the absolute grid $\{m\dy\}$, (b) acquires in blocks
of a fixed count $n_c$, one per traversal, and (c) re-initializes between blocks with a nonzero
latency per \eqref{eq:arm}. A \emph{position-only policy} $P$ is one whose every emit/capture
decision is a function of the encoder position trajectory alone (equivalently, of the crossing
stream $\{T_i(m)\}$), and not of the realized delays $\{\tau_i\}$. A \emph{count cap} is
\emph{surplus}$(s)$ if $n_c{+}s$ crossings are emitted per traversal and \emph{exact} if exactly
$n_c$ are.
\end{definition}

\subsection{The three-layer fidelity contract}

\begin{definition}[Fidelity layers]\label{def:layers}
Let $A\in\R^{R\times n_c}$ be the stored array and let $m_i\in\Z$ be the grid index of block
$i$'s record~$0$, so that $A[i,r]$ was captured at grid index $m_i+r$. Define
\begin{itemize}[leftmargin=1.4em]
  \item \textnormal{\textbf{L1 (intra-block position)}}: within each block the captures lie on
  consecutive grid crossings with unit index spacing---$A[i,r]$ is at position $(m_i+r)\dy$ for
  all $r$;
  \item \textnormal{\textbf{L2 (inter-block phase)}}: the index origins agree across blocks, i.e.\
  $m_i$ is the same constant (w.l.o.g.\ $0$) for all $i$; equivalently the per-block
  \emph{offset} $k_i:=m_i$ vanishes;
  \item \textnormal{\textbf{L3 (assembly)}}: the row index, scan direction, and any reversal maps
  used to place block $i$ at slow-axis row $x_i$ are applied correctly.
\end{itemize}
\emph{Spatial fidelity} is the conjunction L1$\wedge$L2$\wedge$L3: stored index $(i,r)$
corresponds to grid position $(x_i,(k_i+r)\dy)$ with $k_i\equiv0$.
\end{definition}

The offset $k_i$ is the paper's central object. By construction it is the number of emitted
crossings consumed by re-initialization,
\begin{equation}\label{eq:ki-model}
  k_i=\#\{m\ge0:\;T_i(m)\in[t^{\mathrm{start}}_i,\,a_i)\},
\end{equation}
i.e.\ record~$0$ is grid index $k_i$, not $0$, whenever the window opens after the origin
crossing. L3 is standard bookkeeping (a solved engineering layer); we state it for completeness
and do not analyze it further.

\subsection{The separation}

\begin{theorem}[L1 does not imply L2]\label{thm:sep}
There is a grid-locked block-sequential sampler, a position-only policy, and an adversary
$\{\tau_i\}$ under which \textnormal{L1} holds for every block while \textnormal{L2} fails.
Consequently finer grid-locked pitch $\dy$, which sharpens L1, does not by itself deliver spatial
fidelity.
\end{theorem}
\begin{proof}
Take any $P$ that captures $n_c$ consecutive emitted crossings once armed. For every block the
captured positions are $\{(m_i+r)\dy:r=0,\dots,n_c{-}1\}$, a unit-index lattice coset by
\eqref{eq:spacing}: \textbf{L1 holds identically}, for every adversary (this is
Theorem~\ref{thm:L1all}). Now choose the delays so the origins disagree. Let all blocks share a
crossing schedule and pick $\tau_i$ with $\tau_i<T_i(1)-T_i(0)$ for even $i$ (so $k_i=0$) and
$T_i(1)-T_i(0)\le\tau_i<T_i(2)-T_i(0)$ for odd $i$ (so $k_i=1$), which is admissible since each
window is an independent hardware event and \eqref{eq:spacing} makes both intervals nonempty.
Then $k_i$ alternates $0,1,0,1,\dots$: the origins do not agree, so \textbf{L2 fails}. The pitch
$\dy$ enters only through L1 (the spacing), so shrinking it cannot repair the origin
disagreement, which is set by $\{\tau_i\}$. \end{proof}

Theorem~\ref{thm:sep} makes ``position fidelity is not phase fidelity'' precise: the two guarantees
are logically independent, and only the first is delivered by grid-locked triggering.
Section~\ref{sec:sep} sharpens ``fails'' into an impossibility, identifies the minimal primitive
that restores L2, and gives the tight timing tolerance; Section~\ref{sec:rec} asks when a broken
$\{k_i\}$ is recoverable after the fact.

\section{Position--phase separation}\label{sec:sep}

This section develops the paper's core results. All statements are for the count-locked,
position-only class of Definition~\ref{def:policy}.

\begin{theorem}[L1 holds for every policy]\label{thm:L1all}
For every policy $P$ and adversary, all captures lie on grid crossings with correct
unit-index spacing; \textnormal{L1} holds identically.
\end{theorem}
\begin{proof}
A capture occurs only at an emitted crossing, at a grid position $m\dy$. Count-locked
capture of $n_c$ consecutive emitted crossings yields positions
$\{(m_i+r)\dy:r=0,\dots,n_c{-}1\}$---a unit-index lattice coset---independent of $P$ and
$\{\tau_i\}$. The origin $m_i$ is unconstrained; that is the L2 quantity.
\end{proof}

\begin{lemma}[Observational equivalence]\label{lem:obs}
Fix a position-only policy $P$. Any two adversary profiles $\{\tau_i\}$, $\{\tau'_i\}$
present $P$ with identical inputs---the encoder position stream, which the
re-initialization delays do not alter---so $P$ emits the same pulse pattern under both.
Yet the realized origins differ whenever $[t^{\mathrm{start}}_i,t^{\mathrm{start}}_i+\tau_i)$
and $[t^{\mathrm{start}}_i,t^{\mathrm{start}}_i+\tau'_i)$ contain different numbers of
crossings. Hence the stored index-to-position map produced by $P$ is a function of position
alone while the true origin is a function of the adversary: no position-only policy can be
correct for all profiles.
\end{lemma}
\begin{proof}
By Definition~\ref{def:policy} each emit decision depends only on the position trajectory,
which is the same under both profiles; the capture window, however, opens at
$a_i=t^{\mathrm{start}}_i+\tau_i$, so record~$0$ is the first emitted crossing at or after
$a_i$, whose grid index is the crossing count of $[t^{\mathrm{start}}_i,a_i)$---different
across the two profiles by assumption---while the stored column index of record~$r$ is $r$
regardless.
\end{proof}

\begin{theorem}[Impossibility for position-only policies]\label{thm:imposs}
For every position-only policy $P$ there is an adversarial $\{\tau_i\}\in[0,T_{\max}]^R$
such that
\begin{equation}\label{eq:ki}
  k_i=\#\{m\ge0:T_i(m)\in[t^{\mathrm{start}}_i,\,t^{\mathrm{start}}_i+\tau_i)\}
\end{equation}
realizes any profile in $\{0,\dots,K_{\max}\}^R$, $K_{\max}=\lfloor T_{\max}\vmin/\dy\rfloor$
(for an every-crossing emitter; for a position-periodic emitter, any profile reachable by
shifting $\tau_i$ within an inter-emission gap); in particular (when $K_{\max}\ge1$) a
nonzero, non-constant one. Hence no position-only policy guarantees \textnormal{L2}.
\end{theorem}
\begin{proof}
Record~$0$ is the first crossing with $T_i(m)\ge a_i$, so its index is the count
\eqref{eq:ki}. \emph{(Onto.)} As $\tau_i$ grows over $[0,T_{\max}]$, $k_i$ increments by
$1$ each time $\tau_i$ passes a crossing time; by \eqref{eq:spacing} it attains every
integer in $\{0,\dots,K_{\max}\}$, so any profile is realizable. \emph{(Position-only
cannot prevent it.)} By Lemma~\ref{lem:obs} the emit pattern is invariant to $\{\tau_i\}$
while the realized origin is not. Concretely: if $P$ emits at every crossing, record~$0$ is the $k_i$-th crossing; if
on a position-periodic subset, the adversary shifts $\tau_i$ within an inter-emission gap
to change which emitted crossing is first. Either way $k_i$ is adversarially controllable.
\end{proof}

\begin{remark}[Scope of the impossibility---essential]\label{rem:scope}
This does not say grid-locked triggering ``always fails.'' The failure locus is the pulse
stream together with the count-locked window after re-initialization, not the position
trigger, which satisfies L1 (Theorem~\ref{thm:L1all}). A system that arms before the origin
and pins record~$0$ to it (below) does guarantee L2. Position-compare fixes \emph{where}
pulses occur, not \emph{which} pulse becomes record~$0$ once the count-locked window resumes
at the exogenous time $a_i$.
\end{remark}

\begin{theorem}[Synchronization is necessary and sufficient]\label{thm:suff}
Let $P^\star$ withhold block $i$'s first imaged pulse until $\mathrm{armed}_i$ holds and the
stage is at the fixed origin crossing $m=0$. Then $k_i=0$ for all $i$ and every adversary,
so $P^\star$ guarantees \textnormal{L2}. With Theorem~\ref{thm:imposs}, an origin-coupled
synchronization primitive is necessary and sufficient to guarantee \textnormal{L2} at acquisition
(within the count-locked policy class of Definition~\ref{def:policy}).
\end{theorem}
\begin{proof}
The conjunction pins record~$0$ to crossing $m=0$ regardless of $\tau_i$: if
re-initialization completes before the stage reaches $m=0$ the rule waits for $m=0$; if
after, it waits for $\mathrm{armed}_i$ then the designated origin (a fixed pre-declared
position from the absolute encoder). Either way $m_i=0$. Necessity is
Theorem~\ref{thm:imposs}.
\end{proof}

\begin{theorem}[Tight sub-grid threshold]\label{thm:threshold}
Suppose the policy observes the origin crossing time only up to error $|e_i|\le\delta$
(either sign) and uses the \emph{nearest-crossing} rule (record~$0$ = crossing nearest the
perturbed reference). Then \textnormal{L2} is guaranteed for all adversaries iff
\begin{equation}\label{eq:thr}
  \delta<\frac{\dy}{2\vmax}.
\end{equation}
\end{theorem}
\begin{proof}
Let the Voronoi cell of $T_i(0)$ be $(M_-,M_+)$, $M_\pm=\tfrac12(T_i(0)+T_i(\pm1))$. By
\eqref{eq:spacing}, $M_+-T_i(0)=\tfrac12(T_i(1)-T_i(0))\ge\dy/(2\vmax)$, symmetrically on
the left. \emph{(Sufficiency.)} If $\delta<\dy/(2\vmax)$ then
$[T_i(0)-\delta,T_i(0)+\delta]\subset(M_-,M_+)$, so $T_i(0)+e_i$ lies in the cell of
$T_i(0)$ for every $|e_i|\le\delta$; the nearest crossing is $T_i(0)$, $k_i=0$.
\emph{(Necessity.)} If $\delta\ge\dy/(2\vmax)$, take a fastest-speed segment so
$T_i(1)-T_i(0)=\dy/\vmax$ and $M_+-T_i(0)=\dy/(2\vmax)\le\delta$; the adversary sets $e_i$
so $T_i(0)+e_i$ reaches or passes $M_+$, making the nearest crossing $T_i(1)$, $k_i=1$ on
the chosen blocks. (Ties at equality, measure zero, break adversarially.) Hence
\eqref{eq:thr}.
\end{proof}

\begin{remark}[Decision-rule dependence---essential]\label{rem:rule}
The constant depends on the rule. Nearest-crossing gives the tight $\dy/(2\vmax)$. A naive
``first crossing at or after the reference'' rule is not sign-symmetric (a positive error
skips a crossing) and yields only a one-sided guarantee (requiring $e_i\le0$ or an explicit
gate). We adopt nearest-crossing; the qualitative conclusion---resolve the inter-crossing
interval to within a fixed fraction at the fastest speed---is rule-independent. The guarantee
quantifies over admissible speed profiles $v(\cdot)\in[\vmin,\vmax]$; for a fixed, known
profile the same Voronoi argument gives the tight constant $\dy/(2v_0)$ with $v_0$ the speed
at the designated origin crossing.
\end{remark}

\begin{remark}[Physical source of the reference---agnostic]\label{rem:emission}
The theorem constrains only the \emph{resolution} of the origin reference, not its physical
origin. In pulsed-emission modalities (e.g.\ pulsed laser-ultrasound or photoacoustic scanning)
the emission clock is itself a natural timing reference, and gating capture on it is one
implementation of a sub-half-cell reference in the sense of \eqref{eq:thr}; the guarantee is
agnostic to whether that reference comes from an encoder-origin latch, a metadata timestamp, or
the emission clock.
\end{remark}

\subsection{Subharmonic capture under an exact block count}\label{sec:subharm}

A second consequence of count-locking concerns not the offset but the capture
\emph{cadence}. Let the imaged block have exactly $n_c$ crossings (zero surplus, the tight
count cap).

\begin{corollary}[Subharmonic straddle]\label{cor:subharm}
Under a zero-surplus count cap, if the re-initialization delay consumes $N_i\ge1$ crossings
of block $i$ (i.e.\ $k_i\ge1$), the count-locked window cannot complete $n_c$ captures
within the block and must continue past the turnaround into the next sweep---a
\emph{straddle}. If the delay is a sustained fraction $\beta$ of a sweep
($N_i=\beta n_c$, $0<\beta\le1$), the steady-state cadence is one completed batch per
$(1{+}\beta)$ sweeps: each period consumes $\beta n_c$ dead crossings (re-initialization)
plus $n_c$ captured crossings. In particular $\beta\to1$ halves the cadence---one batch per
\emph{two} sweeps, a period-doubling (subharmonic) of the sweep rate. The synchronization
primitive $P^\star$ (Theorem~\ref{thm:suff}) sets $N_i=0$, restoring one batch per sweep.
\end{corollary}
\begin{proof}
The window opens $N_i$ crossings into the sweep and needs $n_c$ pulses, so it captures
$n_c-N_i$ in block $i$ and the remaining $N_i$ after the turnaround, in block $i{+}1$
(opposite direction)---a straddle, with a mislabeled direction and a batch duration
inflated by the turnaround dwell. In steady state each period is [$\beta n_c$ dead][$n_c$ captured]
$=(1{+}\beta)n_c$ crossings $=(1{+}\beta)$ sweeps per completed batch; $\beta\to1$ gives
two. $P^\star$ pins
record~$0$ to crossing $0$ ($N_i=0$), so all $n_c$ crossings are in-block and the batch
completes within one sweep.
\end{proof}

\begin{remark}[Counterintuitive design lesson]
Tightening the block count to \emph{exactly} the number of imaged columns---which removes
the surplus that hides the phase offset---\emph{induces} the subharmonic under any
re-initialization latency. The structural cure is not surplus (which merely re-hides the
offset, cf.\ Theorem~\ref{thm:imposs}) but the origin-coupled handshake
(Theorem~\ref{thm:suff}), which fixes both the cadence and the offset.
\end{remark}

\section{Recoverability of broken phase fidelity}\label{sec:rec}

When L1 holds and L2 fails with unknown integer offsets $\{k_i\}$, can fidelity be restored
post hoc? We give a framework with an exact channel, sufficient conditions, and explicit
impossibility. The per-block shift bound is Robinson--Milanfar \cite{RobinsonMilanfar2004};
C3's contribution is the integer/block structure, the channel trichotomy, the ambiguity
guards, and the abstention certificate.

\begin{proposition}[Metadata channel: exact recovery]\label{prop:meta}
If per-record capture timestamps are available, then
$k_i=\mathrm{round}\bigl((t_i(0)-t_{\mathrm{ref}})/\Delta t_{\mathrm{nom}}\bigr)$ recovers
each offset exactly, and the inverse shift restores L2 (given L1, L3).
\end{proposition}
\begin{proof}
$k_i$ is the grid index of record~$0$ (Definition~\ref{def:layers}); its timestamp divided
by the nominal inter-crossing interval recovers that index when the timestamp resolution is
sub-interval. Applying $A[i,\cdot]\mapsto A[i,\cdot+k_i]$ resets every block's origin.
\end{proof}
This channel is the \emph{strong baseline} any image-only method must beat.

\begin{corollary}[Tight timestamp resolution]\label{cor:metathr}
Let each stored record-$0$ timestamp carry an error of magnitude at most $\varepsilon_t$
(either sign). The rounding decoder of Proposition~\ref{prop:meta} recovers $k_i$ for every
block and every offset profile if and only if
$\varepsilon_t<\Delta t_{\mathrm{nom}}/2$, i.e.\ $\varepsilon_t<\dy/(2v)$ at sweep speed
$v$ (worst case $\dy/(2\vmax)$ over the speed class)---the same half-interval constant as
Theorem~\ref{thm:threshold}, by the same argument.
\end{corollary}
\begin{proof}
Sufficiency: if $\varepsilon_t<\Delta t_{\mathrm{nom}}/2$ the perturbed timestamp remains in
the half-open decoding cell centred on the true crossing time, so rounding returns the true
index, for every error and block. Necessity: at $\varepsilon_t\ge\Delta t_{\mathrm{nom}}/2$
an adversarial error places the timestamp at or past the cell midpoint, and the decoder
returns the adjacent index; ties at the midpoint break adversarially.
\end{proof}

\subsection{Image channel: a registration estimator and its detectability limit}

The image channel is an \emph{image-based registration estimator}: a maximum-likelihood integer-shift
estimator between a block and its grid-consistent neighbour, Robinson--Milanfar/CRLB-limited
(Lemma~\ref{lem:rm}), equipped with the provable identifiability certificate of
Section~\ref{sec:cert}. It is the image-processing method a restorer runs on the formed array; the
detectability threshold below is exactly where that estimator provably succeeds.
Consider block $i$, observed as $A_i$ under additive white Gaussian noise of variance $\sigma^2$
and registered against a grid-consistent neighbour that supplies the local (clean) fast-axis
template $g$ (local slow-axis stationarity); if the neighbour is itself noisy, $\sigma^2$ is read
as the effective comparison variance. Let $\Lam_i:=\sum_j(\partial_y f)^2|_{(i,j)}$ so that
$\sum_j(g[j]-g[j{-}1])^2\approx\Lam_i\dy^2$.

\begin{lemma}[Robinson--Milanfar CRLB \cite{RobinsonMilanfar2004}]\label{lem:rm}
The Fisher information for the continuous fast-axis shift between the two observations is
$\Lam_i/\sigma^2$; any unbiased estimator obeys
$\mathrm{Var}(\widehat{\mathrm{shift}})\ge\sigma^2/\Lam_i$.
\end{lemma}

\begin{proposition}[Exact integer-shift recovery]\label{prop:signal}
With integer candidates in $\{-K_{\max},\dots,K_{\max}\}$ and the local model above, the
maximum-likelihood integer shift equals the truth simultaneously on all $R$ blocks with
probability at least $1-R(2K_{\max}{+}1)\exp(-\Lam_{\min}\dy^2/(8\sigma^2))$,
$\Lam_{\min}=\min_i\Lam_i$. In particular exact recovery holds w.h.p.\ (up to the exponent
constant) when
\begin{equation}\label{eq:c32}
  \frac{\Lam_{\min}\dy^2}{2\sigma^2}>\log\bigl(R(2K_{\max}{+}1)\bigr).
\end{equation}
If $\Lam_i=0$ the shift is unidentifiable (Lemma~\ref{lem:rm}).
\end{proposition}
\begin{proof}
Let $C(\delta)=\tfrac{1}{2\sigma^2}\sum_j(A_i[j]-g[j{-}\delta])^2$. For the true shift $m$
and competitor $m'$, $\mathrm{LLR}=C(m')-C(m)$ has mean $S(m{-}m')/(2\sigma^2)$ and variance
$S(m{-}m')/\sigma^2$, where $S(d)=\sum_j(g[j]-g[j{-}d])^2$. For an adjacent competitor
$S(1)\approx\Lam_i\dy^2$; on a textured, non-periodic band the adjacent shift is the
minimal-separation competitor, so $S(d)\ge S(1)$ for $|d|\ge1$ (the degenerate flat/periodic
cases, where some $S(d)=0$, are exactly Proposition~\ref{prop:ambig}). A sub-Gaussian tail bound gives
single-competitor error $\le\exp(-S(1)/(8\sigma^2))$; a union bound over $R$ blocks and
$2K_{\max}$ competitors yields the stated probability, which \eqref{eq:c32} drives below
$1$. The converse is Lemma~\ref{lem:rm}.
\end{proof}

\begin{remark}[Exact constant]\label{rem:exactconst}
The sub-Gaussian bound is loose by design; with Gaussian noise the adjacent-competitor error is
exactly $\Phi(-\sqrt{\Lam_i\dy^2/(4\sigma^2)})=\Phi(-\sqrt{\rho_i/2})$, where
$\rho_i=\Lam_i\dy^2/(2\sigma^2)$ and $\Phi$ is the standard normal CDF. The simultaneous all-block
recovery threshold is then $\rho^\star$ solving $R\,(2K_{\max})\,\Phi(-\sqrt{\rho^\star/2})=1$; the
synthetic study S3 (Section~\ref{sec:val}) locates the empirical knee at this $\rho^\star$ to within
$3\%$, confirming \eqref{eq:c32} holds with the exact constant, not merely up to order.
\end{remark}

\subsection{Where the image channel fails: ambiguity}

\begin{proposition}[Ambiguity counterexamples]\label{prop:ambig}
The image registration channel is insufficient in two regimes. \textbf{(i) Texture-free:} if
$\partial_y f\equiv0$ on a band, then $S(d)=0$ for all $d$, the likelihood is flat, and
the offset is unidentifiable. \textbf{(ii) Repeated pattern:} if $f(x,\cdot)$ is
$p$-periodic on a band, then $A_i[\cdot]=A_i[\cdot{+}p]$ \emph{identically}, so
$C(\delta)=C(\delta{+}p)$ as functions of the data and the offset is unidentifiable modulo
$p$---regardless of $\Lam_i$, which may be large for a high-contrast grating. In both
regimes image-only recovery must abstain (or defer to metadata/anchors).
\end{proposition}
\begin{proof}
(i) $\partial_y f\equiv0\Rightarrow S(d)\equiv0\Rightarrow\mathbb E[C]$ constant and the
score identically zero (Lemma~\ref{lem:rm}). (ii) $p$-periodicity gives $g[j]=g[j{-}p]$, so
$A_i[j]=A_i[j{-}p]$ on the band and $C(\delta)=C(\delta{+}p)$ as data functions; the
integer argmax is non-unique mod $p$.
\end{proof}

These counterexamples are required: they prevent any
``texture-energy-above-threshold $\Rightarrow$ recoverable'' overclaim and motivate
abstention.

\subsection{Sub-grid residual}

\begin{proposition}[Sub-grid residual, sufficient conditions]\label{prop:subgrid}
A continuous within-grid residual $c_i\in(-\dy/2,\dy/2]$ places the problem in the
classical nonuniform-sampling regime; under bandlimited / shift-invariant-space hypotheses,
Kadec-type and Aldroubi--Gr\"ochenig results \cite{AldroubiGrochenig2001} give
\emph{sufficient} stability and reconstruction conditions. We invoke these only as
sufficient conditions; no ``iff'' transfer is claimed without verifying the hypotheses.
\end{proposition}

\subsection{One condition, three channels: half-cell resolution}\label{sec:halfcell}

Theorem~\ref{thm:threshold}, Corollary~\ref{cor:metathr}, and Proposition~\ref{prop:signal}
are three appearances of a single principle. Call the \emph{phase cell} the inter-crossing
interval in time ($\dy/v$) or, equivalently, the inter-column spacing in position ($\dy$).
Pinning a block's phase requires resolving the phase coordinate to \emph{less than half its
cell}, wherever the pinning is performed:
\begin{itemize}[leftmargin=1.5em]
  \item \emph{at acquisition}: a synchronization reference with $\delta<\dy/(2\vmax)$
  (Theorem~\ref{thm:threshold});
  \item \emph{from metadata}: a timestamp with error below $\Delta t_{\mathrm{nom}}/2=
  \dy/(2v)$ (Corollary~\ref{cor:metathr});
  \item \emph{from the image}: a shift estimate resolving half a column---the
  Robinson--Milanfar-limited standard deviation $\sigma/\sqrt{\Lam_i}<\dy/2$, which is
  \eqref{eq:c32} with the constant replaced by the $\log$-multiplicity term needed for
  simultaneous decoding over $R$ blocks and $2K_{\max}{+}1$ candidates.
\end{itemize}
The three proofs share one shape: correct integer decoding is membership of a perturbed
reference in the Voronoi cell of the truth. As a design heuristic \emph{for this problem} (we claim no
universal or cross-modal budget): \emph{the half-cell margin must be paid by one of three
channels}---at the trigger (prevention), in the metadata (exact recovery), or in the texture
(statistical recovery). A pipeline that pays it in none leaves the phase unrecoverable, and the
regimes of Proposition~\ref{prop:ambig} are exactly where the third channel is unavailable.

\subsection{Calibrated abstention certificate}\label{sec:cert}

The recoverability flags are combined into a per-block decision and calibrated
distribution-free by conformal prediction \cite{VovkGammermanShafer2005,AngelopoulosBates2023,Belhasin2024}.
This is a \emph{supporting framework}, not a new theorem.

\noindent\textbf{Procedure.}
(1) Compute a score $s_i$ (the matching-cost margin of Proposition~\ref{prop:signal}, the
energy $\Lam_i$, and a periodicity-detector statistic guarding
Proposition~\ref{prop:ambig}).
(2) Abstain iff $s_i<\lambda$; the kept-region geometric risk is monotone in $\lambda$.
(3) On a calibration set with known geometry, choose $\widehat\lambda$ by conformal risk
control \cite{ConformalRisk2024} to control the kept-region error at level $\alpha$.

\begin{proposition}[Risk control]\label{prop:cert}
With $\widehat\lambda$ chosen by conformal risk control on an exchangeable calibration set
of size $n$ and a bounded monotone geometric-risk loss $L_\lambda$,
$\mathbb E[L_{\widehat\lambda}(\mathrm{test})]\le\alpha$, distribution-free and
finite-sample (tight to $O(1/n)$), under exchangeability.
\end{proposition}

\begin{remark}[Honest limits]
Proposition~\ref{prop:cert} controls risk \emph{where the method does not abstain}; it does
not identify the unrecoverable set exactly, and requires exchangeability (recalibrate under
distribution shift). Crucially, the guarantee is on the \emph{mean} kept-region risk
$\mathbb{E}[L_{\widehat\lambda}]\le\alpha$ over the exchangeable calibration/test draw---an
average, \emph{not} a per-instance certificate that any single kept block is correct: individual
kept blocks may still err, only their aggregate risk is bounded (S5 shows the achieved mean stays
at or below the target). The certificate is a calibration wrapper around the identifiability
content of Propositions~\ref{prop:signal}--\ref{prop:ambig}; it is not itself a contribution.
\end{remark}

\subsection{Stochastic re-initialization}\label{sec:stoch}

The worst-case results above are complemented by an average-case model fitted to data. The
per-block offset sequence is structured rather than white: in the logged position-only capture
($5\,\mu$m) its same-direction lag-1 autocorrelation is $0.40$--$0.56$ (Table~\ref{tab:rig},
Section~\ref{sec:val}). (An earlier large-offset-regime estimate---$\approx0.62$ at $1\,\mu$m /
$50$\,mm/s---first motivated the model.) This structure motivates a rounded first-order autoregression
$k_i=\mathrm{round}(\mu+z_i)$, $z_i=\rho z_{i-1}+\eta_i$, in place of an independence
assumption. Under this model, recovering $\{k_i\}$ reduces to estimating its increments
(Proposition~\ref{prop:signal}) and accumulating them.

\begin{proposition}[Sufficient recovery / degradation / impossibility]\label{prop:stoch}
Under the AR(1) model: \textnormal{(a)} if every increment clears the
Proposition~\ref{prop:signal} threshold and texture is locally non-periodic, all increments
are exact w.h.p.\ and $\{k_i\}$ is recovered up to a global constant (one anchor resolves
it); \textnormal{(b)} if increments fail with probability $q>0$, the accumulated error is a
random walk of standard deviation $\Theta(\sqrt{i\,q})$, bounded to $\Theta(\sqrt{L\,q})$ by
anchors every $L$ blocks; \textnormal{(c)} texture-free or $p$-periodic bands make increments
unidentifiable (Proposition~\ref{prop:ambig}).
\end{proposition}

\noindent The autoregressive model is a modelling choice motivated by the observed structure; the
formal proof of (a)--(c)---from Proposition~\ref{prop:signal}, an elementary random-walk bound, and
Proposition~\ref{prop:ambig}---is given in the Supplementary Material (Sec.~S-B). Fitting its parameters to
measured offsets at the operating point is a measurement still to be made, and we claim no sharp
transition.

\section{Evidence and validation}\label{sec:val}

The theorems of Sections~\ref{sec:model}--\ref{sec:rec} are self-contained. We first exhibit
the predicted phenomenon on a real position-triggered scanner---the primary evidence---then
corroborate the quantitative predictions with a synthetic sampler. The Section~\ref{sec:real}
values were recorded during the original investigation; the raw acquisition logs are being
re-captured for the data release, and until that re-capture lands they are reported as the
investigation record rather than as deposited data. We mark explicitly the one quantity (an
offset-sequence autocorrelation) that is measured in a different operating regime and used
only as modelling motivation.

\paragraph{What this instrument shows, and what we isolate.}
At submicron pitch the dominant \emph{visible} artifact is an L1 violation---intra-block velocity
ripple warping the within-block spacing (up to $\sim\!1199$\,px at $200$\,mm/s; Supplementary
Table~S5)---which we \emph{bracket out}. This paper characterizes the position-independent L2
violation: a systematic, direction-bimodal displacement of the block index origin ($17.3$ columns on
this system, Table~\ref{tab:rig}), a rigid per-direction shift that misplaces every feature in a block
while the image still looks sharp. Isolating L2 is a deliberate scoping choice (the cleanly separable
layer), not a claim it is the only failure mode at this pitch.

\subsection{The phenomenon on a real position-triggered scanner}\label{sec:real}

The system is a laser-ultrasound C-scanning platform operated as a bidirectional raster scanner
with absolute-encoder position-compare triggering and a count-locked digitizer that
re-initializes between blocks---an instance of the class of Definition~\ref{def:policy}. The
signal in these measurements is the continuous-wave detection response (DC-averaged); the
index-fidelity analysis below is signal-agnostic and does not depend on the detection modality.
The deployed system also included a fixed-function hardware trigger-conditioning stage; the
analysis and claims in this paper do not depend on it. Per-record hardware timestamps and the
row-assembler shift metric give the per-block offset and the per-row grid-phase shift directly. A
theory--measurement map relating each result to its logged observable is in the Supplementary
Material (Table~S3).

\paragraph{Measured per-block offsets (Theorem~\ref{thm:sep}).}
On a logged $5\,\mu$m position-only scan ($903$ blocks) the per-block offset (the timestamp-derived
index of record~$0$) is \emph{bimodal by scan direction}: its mean is $+6.9$ columns on forward
sweeps and $-10.4$ columns on reverse sweeps, a rigid \textbf{$17.3$-column bidirectional index
offset} (one-way $\approx\!8.65$ columns $=43.3\,\mu$m; Table~\ref{tab:rig}, Fig.~\ref{fig:khist}),
while the column spacing is simultaneously on grid. This large, systematic, timestamp-verified
displacement of the block index origin is the paper's real-system L2 evidence---a direct empirical
instance of L1$\,\wedge\,\neg$L2 under a position-only acquisition policy. The \emph{within}-direction
scatter about each mean is only $\approx\!1$ column (std $0.97/1.19$ FWD/REV), which sits near the
image cross-correlation noise floor ($\approx\!1$\,px) and is \emph{not} the headline result. The
surplus-trigger count per block has zero variance, so the offset is a capture-window phase, not
trigger jitter.

\paragraph{The offset is a trigger-grid phase, not velocity ripple (Theorem~\ref{thm:sep}).}
The offset is \emph{independent} of the intra-sweep velocity ripple (timestamp std $\sim\!5.5\,\mu$s
does not move position-compared columns) but scales with the \emph{mean} speed through
$\Phi=\delta v/\dy$, a fixed trigger-to-capture gap $\delta\approx110\,\mu$s acting oppositely on the
two scan directions---an L2 index-origin shift, distinct from the continuous L1 warp. Writing
$k=s\Phi+c$ ($s=\pm1$) separates the direction-reversing swing $\Phi\approx8.65$ columns (the L2
half-swing) from a benign global constant $c\approx-1.78$ columns that does \emph{not} enter the L2
claim; the headline is the $c$-independent inter-direction displacement $2\Phi=17.3$ columns. We
calibrate $\delta$ from the symmetric swing, not the raw means, and attribute no cause to $c$ (a small
forward/reverse speed asymmetry cannot be excluded without the rate-matched capture noted as future
work). The full decomposition and the $\delta$-versus-jitter distinction are in the Supplementary
Material (Sec.~S-C).

\paragraph{Straddle, subharmonic, and threshold (investigation record \textbf{[I]}).}
Two further predictions are corroborated by the June investigation record (raw logs not retained;
detail in Supplementary Table~S4), reported as investigation-record, not logged data. On
an exact-count run one of eight blocks straddles the turnaround and the completed-batch cadence falls to
one row per bidirectional cycle ($\approx449$\,ms), the subharmonic of Corollary~\ref{cor:subharm}, the
worst a catastrophic $\approx22$\,ms straddle; surplus-count and $P^\star$-armed runs do not straddle
(zero of eight), the structural cure of Theorem~\ref{thm:suff}. And at $1\,\mu$m/$200$\,mm/s the
predicted tolerance $\dy/(2\vmax)=2.5\,\mu$s is exceeded by the recorded $\sim\!5\,\mu$s jitter, placing
an unsynchronized reference on the failing side of \eqref{eq:thr} (Theorem~\ref{thm:threshold}),
consistent with the residual offsets.

\paragraph{Position-only acquisition exhibits the offset (Theorem~\ref{thm:imposs}).}
The logged configuration is a position-only acquisition policy: it triggers on encoder position
and applies no origin-coupled reference. It nevertheless shows a nonzero per-block offset on almost
every block (Table~\ref{tab:rig})---exactly the failure Theorem~\ref{thm:imposs} predicts for
position-only policies, and an empirical instance of it (a single instance is consistent with the
theorem, not a proof of it). The offset is recoverable off-line where the image is identifiable:
an image-based registration estimator (maximum-likelihood integer shift, Robinson--Milanfar/CRLB-%
limited) estimates and inverts the per-line shift---the image channel of
Section~\ref{sec:rec}, equipped with an identifiability certificate that abstains on its
texture-free/periodic ambiguity band---which is what Fig.~\ref{fig:beforeafter} illustrates.
\emph{This real-system recovery is qualitative only}: Fig.~\ref{fig:beforeafter}'s panels are
rendered-image proxies that mix L1/L2 and are not used to quantify recovery. The quantitative
recovery-and-abstention evidence---identifiable rows recovered exactly, fabrication on the
unidentifiable band, certified abstention---is the gated synthetic benchmark
(Table~\ref{tab:recovery}, Fig.~\ref{fig:recovery}), not this figure. A rate-matched, logged capture
that decomposes the per-block L2 offset from the intra-block L1 warp on the same region is the
measurement that would sharpen these figures, and is future work. We report this as a single measured instance confirming the
phenomenon on a deployed system, not a characterization: a rate-matched hardware intervention that
isolates a single remedy (a controlled before/after) is out of scope here, and we attribute no
micro-cause to the constant $c$.

\begin{figure}[t]
\centering
\includegraphics[width=\linewidth]{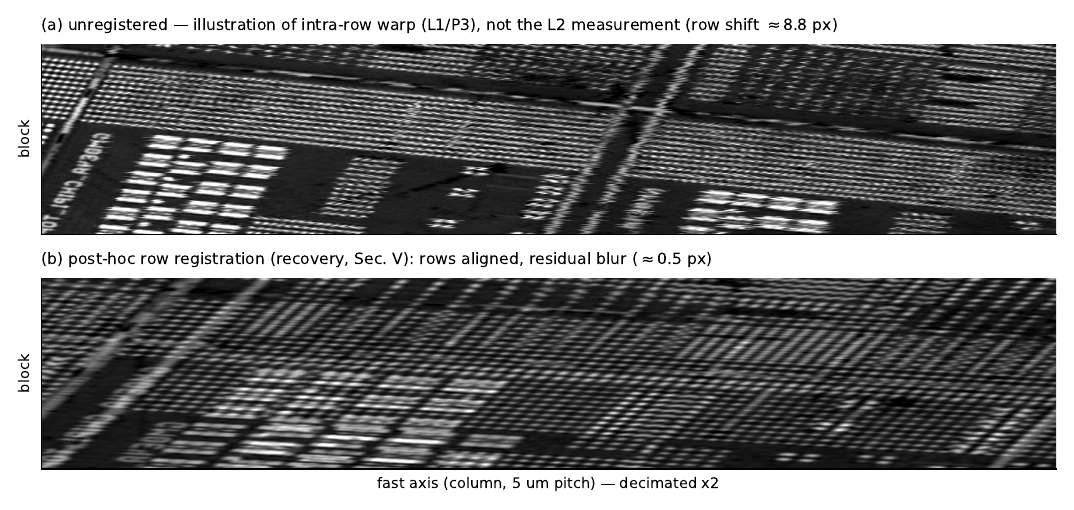}
\caption{Qualitative illustration of the recovery channel on a real capture ($5\,\mu$m pitch).
\textbf{(a) unregistered:} an \emph{illustration of continuous intra-row warp} (an L1/P3 regime we
bracket out, Supplementary Material, Sec.~S-C), \emph{not} the L2 measurement---this state was not
logged, and its per-row distortion varies \emph{within} each row (left/mid/right sub-window shifts
of a single row disagree by $\approx9$\,px, exceeding the $\approx6$\,px median row shift), so it is
warp-dominated rather than a rigid per-block offset. A rendered-image proxy gives same-direction
row-shift std $\approx8.8$\,px. \textbf{(b) image-based registration:} the image channel of
Section~\ref{sec:rec} (a Robinson--Milanfar/CRLB-limited estimator with an identifiability
certificate) estimates and inverts the per-row shift where it is identifiable, aligning the
rows ($\approx0.5$\,px) but leaving a residual blur. The discrete per-block L2 offset is evidenced
separately by the logged $k_i$ of Table~\ref{tab:rig} (a $17.3$-column bidirectional index offset),
not by this figure; the two panels are rendered-image proxies (mixing L1/L2), shown to illustrate the
recovery channel, not to quantify L2.}
\label{fig:beforeafter}
\end{figure}

\paragraph{The logged position-only offset (Table~\ref{tab:rig}, Fig.~\ref{fig:khist}).}
Per-record hardware timestamps give the inter-block index offset directly. In the logged position-only
capture ($5\,\mu$m, $903$ blocks) the systematic bidirectional offset is $17.3$ columns
(FWD $+6.9$, REV $-10.4$; Table~\ref{tab:rig}, Fig.~\ref{fig:khist}), with a within-direction jitter
of std $\approx\!1$ column (max $|k_i|=12$) and lag-1 autocorrelation $0.40$--$0.56$. It was captured
with \emph{full} segment counts at the nominal inter-crossing interval, so no crossings are dropped
and L1's consecutive-crossing premise (Theorem~\ref{thm:L1all}) holds while the per-block index
origin is displaced---the L1$\,\wedge\,\neg$L2 signature. The within-direction autocorrelation is the
measured value used in the AR(1) model of Section~\ref{sec:stoch}, replacing the different-regime
value that only motivates it.

\begin{figure}[t]
\centering
\includegraphics[width=0.86\linewidth]{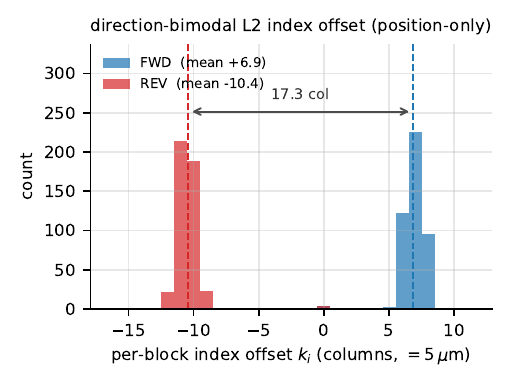}
\caption{Logged per-block index offset $k_i$ for the position-only configuration ($5\,\mu$m),
\emph{bimodal by scan direction}: forward sweeps centre at $+6.9$ columns, reverse at $-10.4$---a
rigid $17.3$-column bidirectional index offset (the L2 signal), about which the within-direction std
is only $\approx\!1$ column. The column spacing is simultaneously on grid. Columns are $5\,\mu$m.}
\label{fig:khist}
\end{figure}

\begin{table}[t]
\centering
\caption{Logged per-block index offset $k_i$, position-only configuration ($903$ blocks, $5\,\mu$m
pitch, $\approx\!390$\,mm/s). Bimodal by scan direction (FWD $+6.9$, REV $-10.4$ columns): a rigid
\textbf{$17.3$-column bidirectional index offset}---the real-system L2 evidence
(Theorems~\ref{thm:sep},~\ref{thm:imposs}), no origin-coupled reference applied. Full interpretation in
Section~\ref{sec:real}; L1 warp in Supplementary Table~S5.}
\label{tab:rig}
\small
\begin{tabular}{lccc}
\hline
scan direction & mean $k_i$ (col) & std $k_i$ (col) & lag-1 autocorr \\
\hline
forward (FWD) & $+6.9$ & $0.97$ & $0.40$ \\
reverse (REV) & $-10.4$ & $1.19$ & $0.56$ \\
\hline
\end{tabular}
\end{table}

\subsection{Synthetic corroboration}\label{sec:synthetic}

A simulated grid-locked block-sequential sampler reproduces the theorems' quantitative
predictions in a controlled, instrument-independent setting. The simulator (Algorithm~A,
released with the paper) samples a known field on the grid, draws a re-initialization delay
per block (worst-case or autoregressive per Section~\ref{sec:stoch}), resumes count-locked
capture at the first crossing past $a_i$, applies the L3 maps, and adds noise, emitting the
array, the true $\{k_i\}$, and per-record timestamps. Five studies map one-to-one to the
results: S1 sweeps the reference jitter $\delta$ across $\dy/(2\vmax)$ and locates the
transition of Theorem~\ref{thm:threshold}; S2 sweeps the arm delay and traces the cadence of
Corollary~\ref{cor:subharm}; S3 sweeps texture energy $\Lam/\sigma^2$ and locates the
detectability threshold \eqref{eq:c32}; S4 uses periodic and texture-free targets to show the
ambiguity of Proposition~\ref{prop:ambig} and that the certificate abstains while the metadata
channel still succeeds; S5 reports conformal coverage against the nominal level. These
experiments test---rather than replace---the predictions: the real-system measurements of
Section~\ref{sec:real} establish that the phenomenon occurs, and the synthetic studies confirm
the predicted thresholds and cadence quantitatively.

\begin{figure*}[t]
\centering
\includegraphics[width=\textwidth]{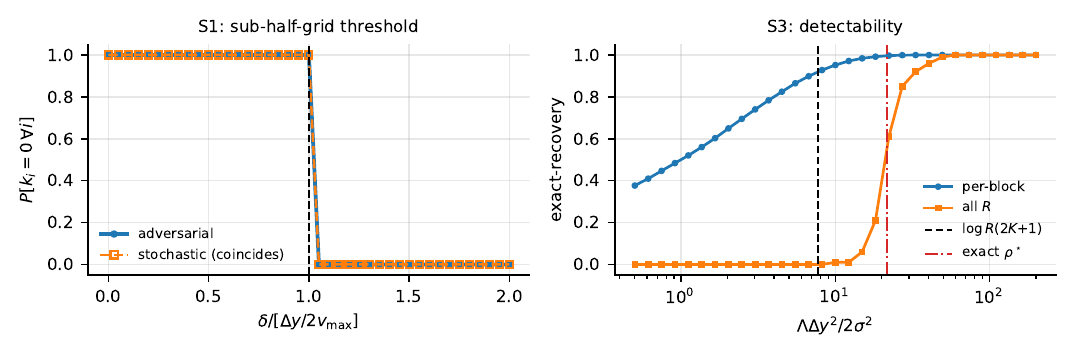}
\caption{Synthetic studies S1 and S3 (Algorithm~A; $\dy{=}1\,\mu$m, $\vmax{=}200$\,mm/s; the S2
cadence study is Supplementary Fig.~S1). \textbf{S1:}
the probability that every block is correctly phased steps from $1$ to $0$ exactly at
$\delta=\dy/(2\vmax)$ (empirical transition within one grid step of the predicted $2.5\,\mu$s;
zero below-threshold failures over $200$ blocks $\times\,100$ seeds), confirming
Theorem~\ref{thm:threshold}; the adversarial and stochastic curves \emph{coincide} (both step at the
same half-cell threshold), so the stochastic curve is drawn dashed with open markers over the solid
adversarial one. \textbf{S3:} the all-block exact-recovery knee sits at $\Lambda\dy^2/(2\sigma^2)\approx22.2$,
matching the exact-Gaussian union-bound prediction $\rho^\star\approx21.7$ (the root of
$R\cdot2K_{\max}\,\Phi(-\sqrt{\rho^\star/2})=1$) to within $3\%$---confirming \eqref{eq:c32} with
its exact constant; $\log R(2K_{\max}{+}1)\approx7.7$ is the order-only prediction.}
\label{fig:synth1}
\end{figure*}

\begin{figure*}[t]
\centering
\includegraphics[width=\textwidth]{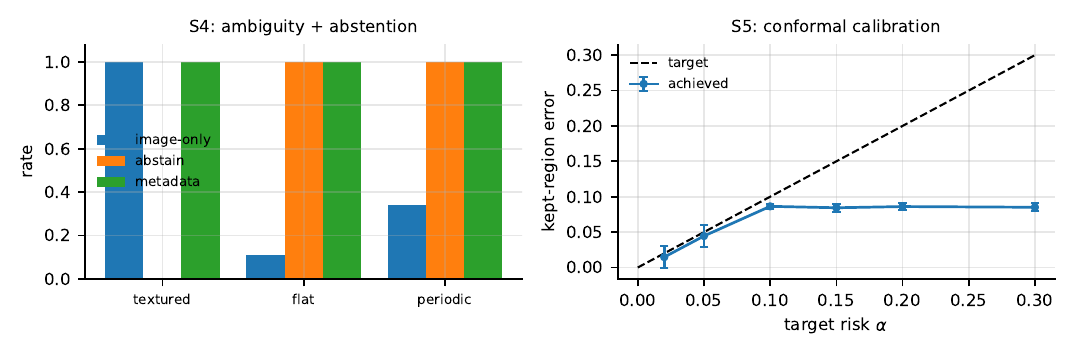}
\caption{Synthetic studies S4--S5. \textbf{S4:} on textured bands
the image channel decodes exactly and the certificate keeps; on texture-free and $p$-periodic
bands (Proposition~\ref{prop:ambig}) the image channel fails ($0.11$/$0.34$ accuracy) and the
certificate abstains ($100\%$), while the metadata channel stays exact everywhere. \textbf{S5:}
the conformally calibrated kept-region error stays at or below the target risk $\alpha$ across
$20$ splits (Proposition~\ref{prop:cert}). The image-restoration reliability evaluation (recovery
vs.\ fabrication vs.\ abstention across baselines) is Fig.~\ref{fig:recovery}.}
\label{fig:synth2}
\end{figure*}

\paragraph{Image-restoration reliability against baselines.}
This is an image-restoration reliability test: which restorers recover identifiable lines, and which
fabricate on the provably-unidentifiable band. On a corrupted image (AR(1) offsets, a catastrophic
straddle, and both texture-free and
$p$-periodic unrecoverable bands), the metadata channel is exact and the image-only ML decoder and its
certificate-gated version recover every identifiable block \emph{exactly} (identifiable-band
RMSE $=0$, a confirmation of the margin-driven recovery of Proposition~\ref{prop:signal}), versus
$43\%$ for rigid registration and column-mean destriping and $84\%$/$33\%$ for two added data-driven
restoration baselines (learned-template and low-rank PCA). The anti-hallucination result: on the
provably-unidentifiable $p$-periodic band \emph{every} offset-estimating method without an abstention
certificate fabricates structure (rigid $0.98$, image-only ML $0.88$, learned-template $0.70$,
smooth-prior $0.63$, PCA $0.18$), whereas only the certificate (and the metadata oracle) abstain
($100\%$)---while still recovering the identifiable rows exactly. Column-mean destriping and low-rank
PCA fabricate little only because they are largely \emph{inert} (they apply no shift or average it
out), \emph{not} because they knowingly decline. The mechanism is objective-driven, not
architecture-specific: any method that optimizes reconstruction accuracy fabricates a shift on the
unidentifiable band, and a deep network trained to the same objective is expected to fabricate at least
as confidently; a controlled deep baseline is out of scope here because it would require domain
retraining, which confounds out-of-distribution failure with the fabrication effect we isolate.
Table~\ref{tab:recovery} summarizes the comparison; the full table with methods and the aggregate/
identifiable-band RMSE columns is in the Supplementary Material (Table~S1).\footnote{The
Supplementary Material also contains the proof of Proposition~\ref{prop:stoch} (Sec.~S-B) and the
June investigation record (Sec.~S-C).}

\begin{table}[t]
\centering
\caption{\textbf{Image-restoration reliability} on the synthetic benchmark (gated;
Fig.~\ref{fig:recovery}; full version Supplementary Table~S1). \emph{exact}: exact-recovery over
the identifiable band; \emph{fab.}: fraction of $p$-periodic-band lines given a nonzero \emph{wrong}
shift; \emph{abst.}: abstention coverage on the texture-free/periodic bands. Destriping/PCA fabricate
little only because they are \emph{inert}, \emph{not} certified abstention. $^\star$ ours;
$^\dagger$ added data-driven baselines.}
\label{tab:recovery}
\small
\begin{tabular}{lccc}
\hline
Method & exact (ident.) & fab.\ (per.) & abst.\ (per.) \\
\hline
metadata-only (oracle)              & $1.00$ & $0.00$ & $0.00$ \\
image-only ML$^\star$               & $1.00$ & $0.88$ & $0.00$ \\
image$+$certificate$^\star$         & $1.00$ & $\mathbf{0.00}$ & $\mathbf{1.00}$ \\
rigid registration                  & $0.43$ & $0.98$ & $0.00$ \\
destriping (col-mean)               & $0.43$ & $0.00$ & $0.00$ \\
smooth-prior                        & $0.67$ & $0.63$ & $0.00$ \\
learned-template$^\dagger$          & $0.84$ & $0.70$ & $0.00$ \\
low-rank (PCA)$^\dagger$            & $0.33$ & $0.18$ & $0.00$ \\
\hline
\end{tabular}
\end{table}

\begin{figure*}[t]
\centering
\includegraphics[width=\textwidth]{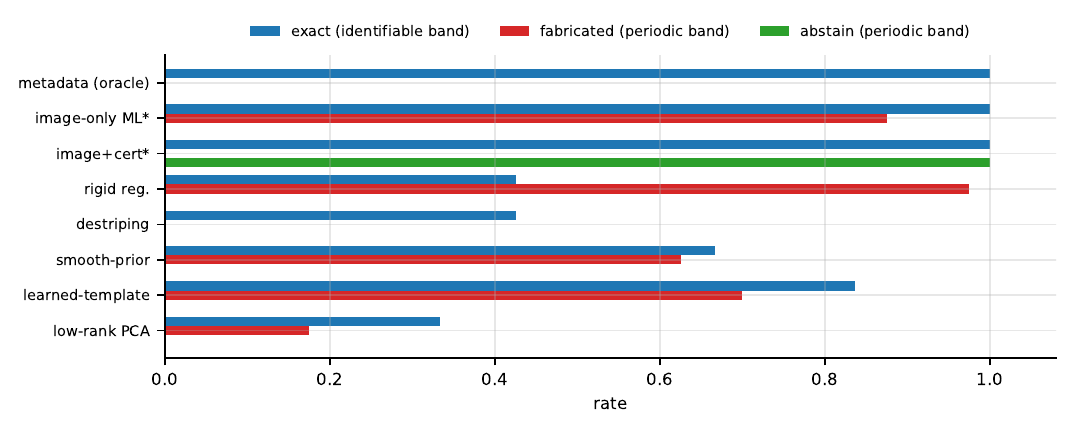}
\caption{\textbf{Image-restoration reliability} on the synthetic benchmark
(Table~\ref{tab:recovery}; full version Supplementary Table~S1). Per restorer: exact-recovery rate on the
identifiable band, \emph{fabrication} rate on the provably-unidentifiable $p$-periodic band, and
abstention coverage on that band. Every restorer without an abstention certificate---rigid
registration, destriping, smooth-prior, and the two data-driven baselines (learned-template,
low-rank PCA)---fabricates structure on the periodic band, whereas the certified method
(\texttt{image+cert}) abstains ($1.00$) and still recovers the identifiable band exactly. Starred
rows ($^\star$) are ours; the low fabrication of destriping/PCA reflects \emph{inertness} (they apply
no shift or average it out), not certified abstention.}
\label{fig:recovery}
\end{figure*}

\subsection{What transfers across modalities}\label{sec:crossmodal}
Re-scoring public data with the paper's own certificate---rolling-shutter registration (TUM, Carla-RS,
Fastec-RS \cite{TUMRollingShutter,LiuCVPR2020}) and cryo-EM SO(2) multireference alignment
\cite{CryoEMMRA,BandeiraMRA}, Supplementary Table~S2---shows the certified-abstention mechanism
carries into every reference-perturbed modality tested: a certificate-free estimator fabricates a
confident position on the unidentifiable regime in every case, whereas the certified method declines.
This follows because the impossibility (Proposition~\ref{prop:ambig}) is distribution-free. What is
modality-specific is how completely it abstains and how tightly it controls kept-region risk.
Scanning holds fully (abstains on essentially all unidentifiable rows; conformal target met).
Rolling-shutter abstains but coverage is loose: the achievable kept-region risk floors at $0.17$--$0.46$,
above the $\alpha{=}0.10$ target, from heavy image-matching tails. Cryo-EM at extreme SNR
\emph{under-abstains}: the confidence certificate flags only $0.39$ of a regime that is
$\approx\!0.79$ unidentifiable, and the kept blocks remain nearly as erroneous as uncertified---an honest
limitation of the confidence rule as SNR$\,\to\,0$. Thus abstention \emph{soundness} is a structural
property that carries, while completeness and coverage tightness are statistical and modality-specific;
no single cross-modal invariant holds. The gated per-modality numbers are in
the Supplementary Material (Table~S2).

\paragraph{Reproducibility.}
Three classes, by what is true at build time. \emph{(i) Synthetic} --- the simulator
(\texttt{scripts/tipkit.py}), study drivers (\texttt{run\_S1..S5.py}, \texttt{run\_recovery.py},
\texttt{make\_figs.py}) and unit tests are released; every synthetic number and figure (the recovery
table is in the Supplementary Material) regenerates from committed
\texttt{data/\{raw,analyze\}\_data/sim/} by one command, gated by \texttt{check\_numbers.py}.
\emph{(ii) July position-only capture} --- Table~\ref{tab:rig}
and Fig.~\ref{fig:khist} are derived from the \emph{committed} raw log at
\texttt{data/logs/2026-07-03\_before/} via
\texttt{extract\_dphase.py}/\texttt{make\_rig\_figs.py}, and are gate-checked;
Fig.~\ref{fig:beforeafter} is an
image proxy rendered from the committed C-scan CSVs (not logged). \emph{(iii) June
investigation record} --- the \textbf{[I]}-tagged rows of Supplementary Tables~S3 and~S5
(Supplementary Material, Sec.~S-C); their raw logs are not retained and are the
subject of the re-capture of Section~\ref{sec:real}. Nothing outside class~(i)/(ii) is claimed as
deposited.

\section{Discussion}\label{sec:disc}

\paragraph{What the results say together.}
Grid-locked triggering fixes \emph{where} samples land (L1) but not \emph{which} crossing becomes
each block's origin (L2): the two are logically independent (Theorem~\ref{thm:sep}), and no
position-only policy can close the gap under adversarial re-initialization
(Theorem~\ref{thm:imposs}). What does close it is a single, quantifiable resource. An
origin-coupled primitive $P^\star$ that arms before the origin and pins record~$0$ to it removes
the offset exactly (Theorem~\ref{thm:suff}); an \emph{imperfect} reference still succeeds, but
only if it resolves the phase to within half a grid cell at the fastest speed,
$\delta<\dy/(2\vmax)$ (Theorem~\ref{thm:threshold}). The same half-cell constant reappears when
the offset is recovered after the fact---from timestamps (Corollary~\ref{cor:metathr}) or from
image texture above the Robinson--Milanfar limit (Proposition~\ref{prop:signal})---so
Section~\ref{sec:halfcell} presents it as a single half-cell condition shared by the three channels.

\paragraph{Design rule: pay the half-cell margin in one channel.}
Pinning a block's phase costs one resolved half-cell, and that cost must be paid by exactly one of three
channels---at the trigger (a synchronization reference, prevention), in the metadata (a sub-interval
timestamp, exact recovery), or in the image texture (a shift estimate resolving half a column,
statistical recovery). This is a design heuristic \emph{within the studied setting}, not a universal
conservation law. Quantitatively, a fixed trigger-to-capture gap $\delta$ displaces the origin by
$\Phi=\delta v/\dy$ cells, so the half-cell condition $\Phi<1/2$ is a \emph{speed--pitch fidelity
envelope}
\begin{equation}\label{eq:envelope}
v<\frac{\dy}{2\delta}\qquad\Longleftrightarrow\qquad \dy>2\delta v,
\end{equation}
a maximum faithful speed for a target pitch; with the calibrated $\delta\approx110\,\mu$s
(Section~\ref{sec:val}) the $5\,\mu$m/$390$\,mm/s capture sits far outside it, consistent with the large
measured offset, so the operating rule is to keep $\delta v<\dy/2$. A pipeline that pays the margin in no
channel leaves the phase unrecoverable---and Proposition~\ref{prop:ambig} shows the texture-free and
periodic bands are exactly where the texture option is unavailable, which is why the certificate
(Section~\ref{sec:cert}) abstains there rather than fabricate.

\paragraph{The counterintuitive cadence lesson.}
Tightening the block count to \emph{exactly} the number of imaged columns---the natural way to
remove the surplus that hides the offset---does not help; it \emph{induces} a subharmonic capture
cadence under any re-initialization latency (Corollary~\ref{cor:subharm}), degrading throughput to
one row per bidirectional cycle in the worst case. Surplus re-hides the offset; only the
origin-coupled handshake fixes both the offset and the cadence. This is a concrete warning for
triggered-DAQ raster design.

\paragraph{Instantiation on a real system.}
Section~\ref{sec:val} exhibits the phenomenon on a position-triggered laser-ultrasound C-scanner
(on-grid spacing with nonzero per-block offsets, the catastrophic straddle, the subharmonic
cadence). The logged configuration is a position-only acquisition policy, and it shows a nonzero
per-block offset on almost every block---an empirical instance of the failure
Theorem~\ref{thm:imposs} predicts for position-only policies. The offset is recoverable off-line
where the image is identifiable, by an image-based registration estimator with an identifiability
certificate (the image channel of Section~\ref{sec:rec}). A rate-matched, logged capture that decomposes the per-block L2 offset from
the intra-block L1 warp on the same region---the measurement flagged in Section~\ref{sec:real}---is
what would turn Fig.~\ref{fig:beforeafter} into a fully quantitative L2 claim.

\paragraph{Scope and limitations.}
The theory is confined to inter-block \emph{phase} (L2). We bracket the continuous intra-block
\emph{warp} (an L1 effect, dominant at submicron pitch on the deployed system, Supplementary Table~S5) as a
separate problem; entangling the two obscures both, and the decomposition is precisely what lets
us report each cleanly. The recovery certificate controls risk only \emph{where it does not
abstain}, and assumes exchangeability of the calibration set (recalibrate under distribution
shift); it does not identify the unrecoverable set exactly. The average-case autoregressive model
(Section~\ref{sec:stoch}) is motivated by an observed offset autocorrelation measured in a
different, large-offset regime and is used only as a modelling choice; fitting it at the operating
point is a measurement still to be made.

\paragraph{Generality of the sampler class.}
The object of study---a grid-locked block-sequential sampler (Definition~\ref{def:policy})---covers
scanning electron and transmission electron microscopy with flyback, atomic-force and
scanning-probe line scanning with re-trace, push-broom and line-scan cameras, optical-coherence
B-scan stacking, and ultrasonic/photoacoustic raster C-scanning. No result depends on the
instrument; each needs only count-locked capture, absolute-grid triggering, and a nonzero
re-initialization latency. The rolling-shutter analogy is instructive by contrast: rolling shutter
is a continuous warp at known row times, whereas the corruption here is an \emph{unknown, discrete}
relabeling of the column origin for a static scene under a known trajectory---invisible to
row-time motion models.

\paragraph{Structural versus statistical.}
Section~\ref{sec:crossmodal} draws the boundary that the certificate exposes: abstention \emph{soundness}
is a structural property, distribution-free and carrying across modalities, whereas how completely the
certificate abstains and how tightly it controls kept-region risk are statistical and modality-specific.
A single cross-modal invariant does \emph{not} hold, and we make no universal-law claim.

\paragraph{Certified abstention and generative hallucination.}
The core principle---a system should decline rather than fabricate when its input is provably
uninformative---connects to a broader concern: hallucination in generative models, where image
restorers and language models alike synthesize confident but unsupported output~\cite{Bhadra2021}. Our certificate is an
existence proof that, where an identifiability structure is computable, fabrication can be provably
prevented rather than merely reduced. The mechanism here derives from the likelihood surface of a
well-posed estimation problem; generative-model confabulation lacks an equivalent computable
identifiability surface, so the mechanism does not transfer directly. We therefore frame this not as a
solution to hallucination but as evidence that certified abstention under provable unidentifiability is
achievable in at least one broad class of problems \cite{ElYanivWiener2010,AngelopoulosI2I2022}, and as
an invitation to identify computable identifiability structures in others.

\section{Conclusion}\label{sec:concl}

Absolute-position triggering does not, by itself, produce a spatially faithful pixel array: it
guarantees intra-block position (L1) but not inter-block phase (L2), and the two are independent
(Theorem~\ref{thm:sep}). No position-only policy can restore L2 (Theorem~\ref{thm:imposs}); an
origin-coupled synchronization primitive is necessary and sufficient to guarantee L2 at acquisition
within the count-locked policy class of Definition~\ref{def:policy} (Theorem~\ref{thm:suff}), and
an imperfect reference suffices iff it resolves half a grid cell at the fastest speed
(Theorem~\ref{thm:threshold}). When phase fidelity is broken, the offsets are exactly recoverable
from metadata, conditionally recoverable from image texture above an explicit detectability
threshold, and provably ambiguous in texture-free and periodic regions, where a distribution-free
certificate abstains rather than hallucinate. A single half-cell condition governs prevention
and both recovery channels in this setting. On a deployed laser-ultrasound C-scanner we measure the predicted failure
directly---a rigid $17.3$-column bidirectional index offset---and the speed--pitch envelope
$v<\dy/(2\delta)$ turns the half-cell condition into an operating rule: keep $\delta v<\dy/2$, or
supply a sub-half-cell reference. The contract and its separation give designers of triggered raster
inspection systems a precise target---resolve the phase cell somewhere in the pipeline---and a clear
account of where, provably, it cannot be recovered at all.

\section*{Data availability}
A complete reproducibility package is released with the paper: the raw C-scan arrays and per-record
hardware-timestamp logs of the July operating point (the \textbf{[J]} data), the synthetic-study
inputs, the simulator, all analysis and figure-generation scripts, unit tests, and the derived data.
Every figure and every reported \textbf{[J]} and synthetic number regenerates from these committed
files by a single command and is checked by an automated numeric gate. The one exception is the June
investigation record (\textbf{[I]}): its raw logs are not retained, so those values are reported as
investigation-record and are the subject of a logged re-capture (Section~\ref{sec:real}). Raw data
and code are available in the released package, archived at Zenodo
(DOI:~10.5281/zenodo.21502909), and from the authors on reasonable request.

\section*{Declaration of generative AI and AI-assisted technologies in the manuscript preparation process}
During the preparation of this work the authors used a generative AI assistant to help draft and edit
prose and to develop analysis and figure-generation scripts. After using this tool, the authors
reviewed and edited the content as needed and take full responsibility for the content of the
publication.

\bibliographystyle{IEEEtran}
\bibliography{refs}

\end{document}